\documentclass[sigconf]{acmart}
\usepackage{bm}
\usepackage{epsfig}
\usepackage{amsmath}
\usepackage{graphicx}
\usepackage{epstopdf}
\usepackage{amsfonts}%
\usepackage{indentfirst}
\usepackage{eufrak}%
\usepackage{subfigure}
\usepackage[ruled]{algorithm2e}

\newtheorem{problem}{\textbf{Problem}}
\newtheorem{definition}{\rm\textbf{Definition}}
\newtheorem{theorem}{\rm\textbf{Theorem}}

\newtheorem{remark}{\rm\textbf{Remark}}
\newtheorem{example}{\rm\textbf{Example}}
\AtBeginDocument{%
  \providecommand\BibTeX{{%
    \normalfont B\kern-0.5em{\scshape i\kern-0.25em b}\kern-0.8em\TeX}}}

\setcopyright{acmcopyright}
\copyrightyear{2021}
\acmYear{2021}
\acmDOI{xxxx/xxxxxx}

\acmConference[ICCPS '21]{ICCPS '21: 12th ACM/IEEE International Conference on
	Cyber-Physical Systems}{May 19--21, 2021}{Nashville, USA}
\acmBooktitle{ICCPS '21: 12th ACM/IEEE International Conference on
	Cyber-Physical Systems,
  May 19--21, 2021, Nashville, USA}
\acmPrice{15.00}
\acmISBN{xxxx-xxxx}

\begin{document}

\title{Rule-based Optimal Control for Autonomous Driving}



 \author{Wei Xiao}
 \affiliation{%
 	\streetaddress{Division of Systems Engineering}
 	\institution{Boston University}
 	\city{Brookline}
 	\state{MA}
 	\postcode{02446}
 }
 \email{xiaowei@bu.edu}

 \author{Noushin Mehdipour}
 \affiliation{%
 	\streetaddress{100 Northern Ave.}
 	\institution{Motional}
 	\city{Boston}
 	\state{MA}
 	\postcode{02210}
 }
 \email{noushin.mehdipour@motional.com}

 \author{Anne Collin}
 \affiliation{%
 	\streetaddress{100 Northern Ave.}
 	\institution{Motional}
 	\city{Boston}
 	\state{MA}
 	\postcode{02210}
 }
 \email{anne.collin@motional.com}

 \author{Amitai Bin-Nun}
 \affiliation{%
 	\streetaddress{100 Northern Ave.}
 	\institution{Motional}
 	\city{Boston}
 	\state{MA}
 	\postcode{02210}
 }
 \email{amitai.binnun@motional.com}

 \author{Emilio Frazzoli}
 \affiliation{%
 	\streetaddress{100 Northern Ave.}
 	\institution{Motional}
 	\city{Boston}
 	\state{MA}
 	\postcode{02210}
 }
 \email{emilio.frazzoli@motional.com}

 \author{Radboud Duintjer Tebbens}
 \affiliation{%
 	\streetaddress{100 Northern Ave.}
 	\institution{Motional}
 	\city{Boston}
 	\state{MA}
 	\postcode{02210}
 }
 \email{radboud.tebbens@motional.com}
 \author{Calin Belta}
 \affiliation{%
	\institution{Motional}
	\city{Boston}
 	\state{MA}
 	\postcode{02210}
}
 \email{calin.belta@motional.com}


 \renewcommand{\shortauthors}{Wei Xiao, Noushin Mehdipour, Anne Collin, Amitai Bin-Nun, Emilio Frazzoli, Radboud Duintjer Tebbens, Calin Belta}

\begin{abstract}

We develop optimal control strategies for Autonomous Vehicles (AVs) that are required to meet complex specifications imposed by traffic laws and cultural expectations of reasonable driving behavior. We formulate these specifications as rules, and specify their priorities by constructing a priority structure.  
We propose a recursive framework, in which the satisfaction of the rules in the priority structure are iteratively relaxed based on their priorities. Central to this framework is an optimal control problem, where convergence to desired states is achieved using Control Lyapunov Functions (CLFs), and safety is enforced through Control Barrier Functions (CBFs). We also show how the proposed framework can be used for after-the-fact, pass / fail evaluation of trajectories - a given trajectory is rejected if we can find a controller producing a trajectory that leads to less violation of the rule priority structure. We present case studies with multiple driving scenarios to demonstrate the effectiveness of the proposed framework.

\end{abstract}

\ccsdesc[500]{Computer systems organization~Robotic control}
\ccsdesc[300]{Hardware~Safety critical systems}
\ccsdesc[100]{Computing methodologies~Computational control theory}

\keywords{Autonomous driving, Lyapunov methods, Safety, Priority Structure}


\maketitle

\section{Introduction}
\label{sec:intro}
With the development and integration of cyber physical and safety critical systems in various engineering disciplines, there is an increasing need for computational tools for verification and control of such systems according to rich and complex specifications. A prominent example is 
autonomous driving, which received a lot of attention during the last decade. 
Besides common objectives in optimal control problems, such as minimizing the energy consumption and travel time, and constraints on control variables, such as maximum acceleration, autonomous vehicles (AVs) should follow complex and possibly conflicting traffic laws with different priorities. They should also meet cultural expectations of reasonable driving behavior \cite{nolte2017towards,shalev2017formal,parseh2019pre,ulbrich2013probabilistic,qian2014priority,iso2019pas,Collin2020}. For example, an AV has to avoid collisions with other road users (high priority), drive faster than the minimum speed limit (low priority), and maintain longitudinal clearance with the lead car (medium priority). We formulate these behavior specifications as a set of rules with a priority structure that captures their importance \cite{Censi2020}. 

To accommodate the rules, we formulate the AV control problem as 
an optimal control problem, in which the satisfaction of the rules and some vehicle limitations are enforced by Control Barrier Functions (CBF) \cite{Ames2017}, and convergence to desired states is achieved through Control Lyapunov Functions \cite{Freeman1996}. 
To minimize violation of the set of rules, we formulate iterative rule relaxation according to the pre-order on the rules.

Control Lyapunov Functions (CLFs) \cite{Freeman1996,Artstein1983} have been used to stabilize systems to desired states. CBFs enforce set forward-invariance \cite{Tee2009,Wisniewski2013}, and have been adopted to guarantee the satisfaction of safety requirements \cite{Ames2017,wang2017safety,Lindemann2018}. In \cite{Ames2017,Glotfelter2017}, the constraints induced by
CBFs and CLFs were used to formulate quadratic programs (QPs) that could be solved in real time to stabilize affine control systems while optimizing quadratic costs and satisfying state and control constraints. The main limitation of this approach is that the resulting QPs can easily become infeasible, especially when bounds on control inputs are imposed in addition to the safety specifications and the state constraints, or for constraints with high relative degree \cite{Xiao2019}. Relaxations of the (hard) CLF \cite{Aaron2012,Ames2017}
and CBF \cite{Xiao2019} constraints 
have been proposed to address this issue. 

The approaches described above do not consider the (relative) importance of the safety constraints during their relaxations. With particular relevance to the application considered here, AVs often deal with situations where there are conflicts among some of the traffic laws or other requirements. For instance, consider a scenario where a pedestrian walks to the lane in which the AV is driving - it is impossible for the AV to avoid a collision with the pedestrian or another vehicles, stay in lane, and drive faster than the minimum speed limit at the same time. 
Given the relative priorities of these specifications, a reasonable AV behavior would be to avoid a collision with the pedestrian or other vehicles (high priority), and instead violate low or medium priority rules, e.g., by reducing speed to a value lower than the minimum speed limit, and/or deviating from its lane. The maximum satisfaction and minimum violation of a set of rules expressed in temporal logic were studied in \cite{dimitrova2018maximum,tuumova2013minimum} and solved by assigning positive numerical weights to formulas based on their priorities \cite{tuumova2013minimum}. In \cite{Censi2020}, the authors proposed \emph{rulebooks}, a framework in which relative priorities were captured by a pre-order. In conjunction with rule violation scores, rulebooks were used to rank vehicle trajectories. These works do not consider the vehicle dynamics, or assume very simple forms, such as finite transition systems. The violation scores are example - specific, or are simply the quantitative semantics of the logic used to formulate the rules. In their current form, they capture worst case scenarios and are non-differentiable, and   
cannot be used for generating controllers for realistic vehicle dynamics. 

In this paper, we draw inspiration from Signal Temporal Logic (STL) \cite{Maler2004} to formalize traffic laws and other driving rules and to quantify the degree of violation of the rules by AV trajectories. We build on the rulebooks from \cite{Censi2020} to construct a rule priority structure. The main contribution of this paper is an iterative procedure that uses the rule priority to determine a control strategy that minimizes rule violation globally. We show how this procedure can be adapted to develop transparent and reproducible rule-based pass/fail evaluation of AV trajectories in test scenarios. 
Central to these approaches is an optimization problem based on \cite{Xiao2019}, which uses detailed vehicle dynamics, ensures the satisfaction of ``hard" vehicle limitations (e.g., acceleration constraints), and can accommodate rule constraints with high relative degree. Another key contribution of this work is the formal definition of a speed dependent, optimal over-approximation of a vehicle footprint that ensures differentiability of clearance-type rules, which enables the use of powerful approaches based on CBF and CLF. Finally, we use and test the proposed architecture and algorithms were implemented in a user-friendly software tool in various driving scenarios.

\section{Preliminaries}
\label{sec:pre}

\subsection{Vehicle Dynamics}\label{sec:vd}
Consider an affine control system given by:\vspace{-3pt}
\begin{equation}\label{eqn:affine}
\dot{\bm{x}}=f(\bm x)+g(\bm x)\bm u, 
\vspace{-3pt}
\end{equation}
where $\bm x\in X\subset\mathbb{R}^{n}$ ($X$ is the state constraint set), $\dot{()}$ denotes differentiation with respect to time, 
$f:\mathbb{R}^{n}\rightarrow\mathbb{R}^{n}$
and $g:\mathbb{R}^{n}\rightarrow\mathbb{R}^{n\times q}$ are globally
Lipschitz, and $\bm u\in U\subset\mathbb{R}^{q}$, where $U$ is the control constraint set
defined as:
\begin{equation}
U:=\{\bm u\in\mathbb{R}^{q}:\bm u_{min}\leq\bm u\leq\bm u_{max}\},
\label{eqn:control}%
\end{equation}
with $\bm u_{min},\bm u_{max}\in\mathbb{R}^{q}$, and the inequalities are
interpreted componentwise. We use $\bm{x}(t)$ to refer to a trajectory of (\ref{eqn:affine}) at a specific time $t$, and we use $\mathcal{X}$ to denote a whole trajectory starting at time 0 and ending at a final time specified by a scenario.  
Note that most vehicle dynamics, such as ``traditional" dynamics defined with respect to an inertial frame \cite{Ames2017} and dynamics defined along a given reference trajectory \cite{Rucco2015} (see (\ref{eqn:vehicle})) are in the form (\ref{eqn:affine}). Throughout the paper, we will refer to the vehicle with dynamics given by (\ref{eqn:affine}) as {\em ego}.

\begin{definition}
\label{def:forwardinv}(\textit{Forward invariance} \cite{Nguyen2016}) A set $C\subset\mathbb{R}^{n}$ is forward invariant for
system (\ref{eqn:affine}) if $\bm x(0) \in C$ implies $\bm x(t)\in C,$ $\forall t\geq0$.
\end{definition}

\begin{definition}
\label{def:relative} (\textit{Relative degree} \cite{Nguyen2016}) The relative degree of a
(sufficiently many times) differentiable function $b:\mathbb{R}^{n}%
\rightarrow\mathbb{R}$ with respect to system (\ref{eqn:affine}) is the number
of times it needs to be differentiated along its dynamics (Lie derivatives) until the control
$\bm u$ explicitly shows in the corresponding derivative.
\end{definition}

In this paper, since function $b$ is used to define a constraint $b(\bm
x)\geq0$, we will also refer to the relative degree of $b$ as the relative
degree of the constraint. 

\subsection{High Order Control Barrier Functions}
\begin{definition}
\label{def:classk} (\textit{Class $\mathcal{K}$ function} \cite{Khalil2002}) A
continuous function $\alpha:[0,a)\rightarrow[0,\infty), a > 0$ is said to
belong to class $\mathcal{K}$ if it is strictly increasing and $\alpha(0)=0$.
\end{definition}
Given $b:\mathbb{R}^{n}\rightarrow\mathbb{R}$ and a constraint $b(\bm x)\geq0$ with relative
degree $m$, we define $\psi_{0}(\bm
x):=b(\bm x)$ and a sequence of functions $\psi_{i}:\mathbb{R}%
^{n}\rightarrow\mathbb{R},i\in\{1,\dots,m\}$:
\vspace{-2pt}
\begin{equation}
\begin{aligned} \psi_i(\bm x) := \dot \psi_{i-1}(\bm x) + \alpha_i(\psi_{i-1}(\bm x)),i\in\{1,\dots,m\}, \end{aligned} \label{eqn:functions}%
\end{equation}
where $\alpha_{i}(\cdot),i\in\{1,\dots,m\}$ denotes a $(m-i)^{th}$ order
differentiable class $\mathcal{K}$ function. We further define a sequence of sets $C_{i}, i\in\{1,\dots,m\}$ associated
with (\ref{eqn:functions}) in the following form:
\begin{equation}
\label{eqn:sets}\begin{aligned} C_i := \{\bm x \in \mathbb{R}^n: \psi_{i-1}(\bm x) \geq 0\}, i\in\{1,\dots,m\}. \end{aligned}
\end{equation}

\begin{definition}
\label{def:hocbf} (\textit{High Order Control Barrier Function (HOCBF)}
\cite{Xiao2019}) Let $C_{1}, \dots, C_{m}$ be defined by (\ref{eqn:sets}%
) and $\psi_{1}(\bm x), \dots, \psi_{m}(\bm x)$ be defined by
(\ref{eqn:functions}). A function $b: \mathbb{R}^{n}\rightarrow\mathbb{R}$ is
a High Order Control Barrier Function (HOCBF) of relative degree $m$ for
system (\ref{eqn:affine}) if there exist $(m-i)^{th}$ order differentiable
class $\mathcal{K}$ functions $\alpha_{i},i\in\{1,\dots,m-1\}$ and a class
$\mathcal{K}$ function $\alpha_{m}$ such that 
\begin{equation}
\label{eqn:constraint}\begin{aligned} \sup_{\bm u\in U}[L_f^{m}b(\bm x) + L_gL_f^{m-1}b(\bm x)\bm u + S(b(\bm x)) \\+ \alpha_m(\psi_{m-1}(\bm x))] \geq 0, \end{aligned}
\end{equation}
for all $\bm x\in C_{1} \cap,\dots, \cap C_{m}$. 
$L_{f}^{m}$ ($L_{g}$) denotes Lie derivatives along
$f$ ($g$) $m$ (one) times, and $S(\cdot)$ denotes the remaining Lie derivatives
along $f$ with degree less than or equal to $m-1$ (see \cite{Xiao2019} for more details).
\end{definition}

The HOCBF is a general form of the relative degree $1$ CBF \cite{Ames2017},
\cite{Glotfelter2017}, \cite{Lindemann2018} (setting $m=1$ reduces the HOCBF to
the common CBF form in \cite{Ames2017}, \cite{Glotfelter2017}, \cite{Lindemann2018}), and is also a general form of the exponential CBF
\cite{Nguyen2016}. 

\begin{theorem}
\label{thm:hocbf} (\cite{Xiao2019}) Given a HOCBF $b(\bm x)$ from Def.
\ref{def:hocbf} with the associated sets $C_{1}, \dots, C_{m}$ defined
by (\ref{eqn:sets}), if $\bm x(0) \in C_{1} \cap,\dots,\cap C_{m}$,
then any Lipschitz continuous controller $\bm u(t)$ that satisfies
(\ref{eqn:constraint}) $\forall t\geq0$ renders $C_{1}\cap,\dots,
\cap C_{m}$ forward invariant for system (\ref{eqn:affine}).
\end{theorem}


\begin{definition}\label{def:clf} \textit{(Control Lyapunov Function (CLF) \cite{Aaron2012})} 
A continuously differentiable function $V :\mathbb{R}^{n}\rightarrow\mathbb{R}_{\ge0}$ is an exponentially stabilizing control Lyapunov function (CLF) if there exist positive constants $c_1 >0, c_2 > 0, c_3 > 0$ such that $\forall 
\bm x\in X$, $c_{1}||\bm x||^{2} \leq V(\bm x)
\leq c_{2} ||\bm x||^{2}$, the following holds:
\begin{equation}\label{eqn:CLF}
\inf_{\bm u\in U} \lbrack L_{f}V(\bm x)+L_{g}V(\bm x)
\bm u + c_{3}V(\bm x)\rbrack\leq0.
\end{equation}
\end{definition} 

\begin{theorem} [\cite{Aaron2012}] \label{thm:clf}
	 Given a CLF as in Def. \ref{def:clf}, any Lipschitz continuous controller $ \bm u(t),\forall t\geq 0$ that satisfies (\ref{eqn:CLF}) 
	exponentially stabilizes system (\ref{eqn:affine}) to the origin.
\end{theorem}





Recent works \cite{Ames2017},\cite{Lindemann2018},\cite{Nguyen2016} combined
CBFs and CLFs with quadratic costs to formulate an optimization problem that stabilized a system using CLFs subject to safety constraints given by CBFs. In this work, we follow a similar approach. Time is discretized and CBFs and CLFs constraints are considered at each discrete time step. Note that these constraints are linear in control since the state value is fixed at the beginning of the discretization interval. Therefore, in every interval, the optimization problem is a QP 
. The optimal control obtained by solving each QP is applied at the current time step and held constant for the whole interval. The next state is found by integrating the dynamics (\ref{eqn:affine}). The usefulness of this approach is conditioned upon the feasibility of the QP at every time step. In the case of constraints with high relative degrees, which are common in autonomous driving, the CBFs can be replaced by HOCBFs. 

\subsection{Rulebooks}
\label{sub-sec:rulebooks}

As defined in \cite{Censi2020}, 
a {\em rule} specifies a desired behavior for autonomous vehicles. Rules can be derived from traffic laws, local culture, or consumer expectation, e.g., ``stay in lane for all times", ``maintain clearance from pedestrians for all times", ``obey the maximum speed limit for all times", ``reach the goal". A \textit{rulebook} as introduced in \cite{Censi2020} defines a priority on rules by imposing a pre-order that can be used to rank AV trajectories:

\begin{definition} \label{def:rb}(Rulebook \cite{Censi2020})
A rulebook is a tuple $\langle R,\leq\rangle$, where $R$ is a finite set of rules and $\leq$ is a pre-order on $R$.
\end{definition}

A rulebook can be represented by a directed graph, where each node is a rule and an edge between two rules means that the first rule has higher priority than the second. Formally, $r_1\rightarrow r_2$ in the graph means that $r_1\leq r_2$ ($r_2 \in R$ has a higher priority than $r_1\in R$). Note that, using a pre-order, two rules can be in one of three relations: comparable (one has a higher priority than the other), incomparable, or equivalent (each has a higher priority than the other).  
\vspace{-4pt}
\begin{example}\label{exm:pre-order}
Consider the rulebook shown in Fig. \ref{fig:rb}, which consists of 6 rules. In this example, $r_1$ and $r_2$ are incomparable, and both have a higher priority than $r_3$ and $r_4$. Rules $r_3$ and $r_4$ are equivalent ($r_3\leq r_4$ and $r_4\leq r_3$), but are incomparable to $r_5$. Rule $r_6$ has the lowest priority among all rules. 
\vspace{-4pt}
\begin{figure}[ptbh]
\centering
\includegraphics[scale=0.3]{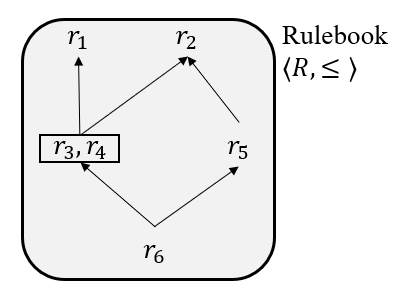}
\vspace{-4pt} \caption{Graphical representation of a rulebook $\langle R,\leq\rangle$. 
}%
\label{fig:rb}%
\end{figure}
\end{example}
\vspace{-4pt}
Rules are evaluated over vehicle trajectories (i.e., trajectories of system (\ref{eqn:affine})). A {\em violation metric} is a function specific to a rule that takes as input a trajectory and outputs a \emph{violation score} that captures the degree of violation of the rule by the trajectory \cite{Censi2020}. For example, if the AV crosses the lane divider and  reaches within the left lane by a maximum distance of 1m along a trajectory, then the violation score for that trajectory against the ``stay in lane for all times" rule can be 1m. 


\section{Problem Formulation}
\label{sec:prob}

For a vehicle with dynamics given by (\ref{eqn:affine}) and starting at a given state ${\bm x}(0)=\bm x_0$, consider an optimal control problem in the form:
\vspace{-3pt}
\begin{equation}\label{eqn:gcost}
\min_{\bm u(t)} \int_{0}^{T}J(||\bm u(t)||)dt,
\end{equation}
where $||\cdot||$ denotes the 2-norm of a vector, $T > 0$ denotes a bounded final time, and $J$ is a strictly increasing function of its argument (e.g., an energy consumption function $J(||\bm u(t)||) = ||\bm u(t)||^2$). We consider the following additional requirements:

\textbf{Trajectory tracking}: We require the vehicle to stay as close as possible to a desired {\em reference trajectory} $\mathcal{X}_r$ (e.g., middle of its current lane).

\textbf{State constraints}: We impose a set of constraints (componentwise) on the state of system (\ref{eqn:affine}) in the following form:
\begin{equation}\label{eqn:state}
\bm x_{min} \leq \bm x(t)\leq \bm x_{max}, \forall t\in[0,T],
\end{equation}
where $\bm x_{max}: = (x_{max,1},x_{max,2},\dots,x_{max,n})\in \mathbb{R}^n$ and $\bm x_{min}: = (x_{min,1},x_{min,2},\dots,x_{min,n})\in \mathbb{R}^n$ denote the maximum and minimum state vectors, respectively. Examples of such constraints for a vehicle include maximum acceleration, maximum braking, and maximum steering rate.  

\textbf{Priority structure:}  We require the system trajectory $\mathcal{X}$ of (\ref{eqn:affine}) starting at $\bm x(0)=\bm x_0$ to satisfy a priority structure $\langle R,\sim_p,\leq_p\rangle$, i.e.:
\begin{equation}\label{eqn:rulebook-sat}
\mathcal{X}\models \langle R,\sim_p,\leq_p\rangle, 
\end{equation}
where $\sim_p$ is an equivalence relation over a finite set of rules $R$ and $\leq_p$ is a total order over the equivalence classes.
Our priority structure is related to the rulebook from Sec. \ref{sub-sec:rulebooks}, but it requires that any two rules from $R$ are either comparable or equivalent (see Sec. \ref{sec:prio-struc} for a formal definition).  Informally, this means that $\mathcal{X}$ is the ``best" trajectory that (\ref{eqn:affine}) can produce, considering the violation metrics of the rules in $R$ and the priorities captured by $\sim_p$ and $\leq_p$. 
A formal definition for a priority structure and its satisfaction will be given in Sec. \ref{sec:prio-struc}. 

\textbf{Control bounds}: We impose control bounds as given in (\ref{eqn:control}). Examples include jerk and steering acceleration.  

Formally, we can define the optimal control problem as follows:

\vspace{1mm}
\begin{problem}\label{prob:main}
	Find a control policy for system (\ref{eqn:affine}) such that the objective function in (\ref{eqn:gcost}) is minimized, and the trajectory tracking, state constraints (\ref{eqn:state}), priority structure $\langle R,\sim_p,\leq_p\rangle$, and control bounds (\ref{eqn:control}) are satisfied by the generated trajectory given $\bm x(0)$. 
\end{problem}

Our approach to Problem \ref{prob:main} can be summarized as follows: We use CLFs for tracking the reference trajectory $\mathcal{X}_r$ and HOCBFs to implement the state constraints (\ref{eqn:state}). For each rule in $R$, we define violation metrics. We show that satisfaction of the rules can be written as forward invariance for sets described by differential functions, and enforce them using HOCBFs.  The control bounds (\ref{eqn:control}) are considered as constraints. We provide an iterative solution to Problem \ref{prob:main}, where each iteration involves solving a sequence of QPs. In the first iteration, all the rules from $R$ are considered. If the corresponding QPs are feasible, then an optimal control is found. Otherwise, we iteratively relax the satisfaction of rules from subsets of $R$ based on their priorities, and minimize the corresponding relaxations by including them in the cost function.

\section{Rules and priority structures}
\label{sec:trb}

In this section, we extend the rulebooks from \cite{Censi2020} by formalizing the rules and defining violation metrics. We introduce a {\em priority structure}, in which all rules are comparable, and it is particularly suited for the hierarchical control framework proposed in Sec. \ref{sec:optim-rule-approx}. 
\subsection{Rules} 
In the definition below, an {\em instance} $i\in S_p$ is a traffic participant or artifact that is involved in a rule, where $S_p$ is the set of all instances involved in the rule. For example, in a rule to maintain clearance from pedestrians, a pedestrian is an instance, and there can be many instances encountered by ego in a given scenario. Instances can also be traffic artifacts like the road boundary (of which there is only one), lane boundaries, or stop lines. 

\begin{definition} (Rule)\label{def:rule}
A rule is composed of a statement and three violation metrics.
A statement is a formula that is required to be satisfied for all times. A formula is inductively defined as:
\begin{equation} \label{eqn:task}
\varphi := \mu\vert \neg \varphi \vert \varphi_1\wedge \varphi_2,
\end{equation}
where $\varphi,\varphi_1,\varphi_2$ are formulas, $\mu :=(h(\bm x)\geq 0)$ is a predicate on the state vector $\bm x$ of system (\ref{eqn:affine}) with $h:\mathbb{R}^n\rightarrow \mathbb{R}$. $\wedge, \neg$ are Boolean operators for conjunction and negation, respectively. The three violation metrics for a rule $r$ are defined as:
 \begin{enumerate}
     \item instantaneous violation metric $\varrho_{r,i}(\bm x(t)) \in [0,1],$
     \item instance violation metric $\rho_{r,i}(\mathcal{X})\in [0,1]$, and
     \item total violation metric $P_{r}(\mathcal{X})\in [0,1]$,
 \end{enumerate}
  where $i$ is an instance, $\bm{x}(t)$ is a trajectory at time $t$ and $\mathcal{X}$ is a whole trajectory of ego. 
 The instantaneous violation metric $\varrho_{r,i}(\bm x(t))$
 quantifies violation by a trajectory at a specific time $t$ with respect to a given instant $i$. The instance violation metric $\rho_{r,i}(\mathcal{X})$ captures violation with respect to a given instance $i$ over the whole time of a trajectory, and is obtained by aggregating $\varrho_{r,i}(\bm x(t))$ over the entire time of a trajectory $\mathcal{X}$. 
 The total violation metric $P_{r}$ is the aggregation of the instance violation metric $\rho_{r,i}(\mathcal{X})$ over all instances $i\in S_p$. 
 \end{definition}
 
 The aggregations in the above definitions can be implemented through selection of a maximum or a minimum, integration over time, summation over instances, or by using general $L_p$ norms. A zero value for a violation score shows satisfaction of the rule. A strictly positive value denotes violation - the larger the score, the more ego violates the rule. 
 Throughout the paper, for simplicity, we use $\varrho_{r}$ and $\rho_{r}$ instead of $\varrho_{r,i}$ and $\rho_{r,i}$ if there is only one instance. Examples of rules (statements and violations metrics and scores) are given in Sec. \ref{sec:case} and in the Appendix. 

We divide the set of rules into two categories: (1) {\em clearance rules} - safety relevant rules
enforcing that ego maintains a minimal distance to other traffic participants and to the side of the road or lane (2) {\em non-clearance rules} - rules that that are not contained in the first category, such as speed limit rules. In Sec. \ref{sec:rule-approx}, we provide a general methodology to express clearance rules as inequalities involving differentiable functions, which will allow us to enforce their satisfaction using HOCBFs. 

 \begin{remark}
 The violation metrics from Def. \ref{def:rule} are inspired from Signal Temporal Logic (STL) robustness \cite{Maler2004,donze,mehdipour2019agm}, which quantifies how a signal (trajectory) satisfies a temporal logic formula. In this paper, we focus on rules that we aim to satisfy for all times. Therefore, the rules in (\ref{eqn:task}) can be seen as (particular) STL formulas, which all start with an ``always" temporal operator (omitted here).  
 \end{remark}

\subsection{Priority Structure}
\label{sec:prio-struc}


The pre-order rulebook in Def. \ref{def:rb} defines a ``base" pre-order that captures relative priorities of some (comparable) rules, which are often similar in different states and countries. 
A pre-order rulebook can be made more precise for a specific legislation by adding rules and/or priority relations through priority refinement, rule aggregation and augmentation \cite{Censi2020}. This can be done through empirical studies or learning from local data to construct a total order rulebook. To order trajectories, authors of \cite{Censi2020} enumerated all the total orders compatible with a given pre-order. In this paper, motivated by the hierarchical control framework
described in Sec. \ref{sec:optim-rule-approx}, we require that any two rules are in a relationship, in the sense that they are either equivalent or comparable with respect to their priorities.

\begin{definition} [Priority Structure]
A priority structure is a tuple $\langle R,\sim_p,\leq_p\rangle$, where $R$ is a finite set of rules, $\sim_p$ is an equivalence relation over $R$, and $\leq_p$ is a total order over the set of equivalence classes determined by $\sim_p$. 
\end{definition}
Equivalent rules (i.e., rules in the same class) have the same priority. Given two equivalence classes $\mathcal{O}_1$ and $\mathcal{O}_2$ with $\mathcal{O}_1\leq_p \mathcal{O}_2$, every rule $r_1\in \mathcal{O}_1$ has lower priority than every rule $r_2\in \mathcal{O}_2$. Since $\leq_p$ is a total order, any two rules $r_1,r_2\in R$ are comparable, in the sense that exactly one of the following three statements is true: (1) $r_1$ and $r_2$ have the same priority, (2) $r_1$ has higher priority than $r_2$, and (3) $r_2$ has higher priority than $r_1$. 
Given a priority structure $\langle R,\sim_p,\leq_p\rangle$, we can assign numerical (integer) priorities to the rules. We assign priority 1 to the equivalence class with the lowest priority, priority 2 to the next one and so on. 
The rules inside an equivalence class inherit the priority from their equivalence class. Given a priority structure $\langle R,\sim_p,\leq_p\rangle$ and violation scores for the rules in $R$, we can compare trajectories:

\begin{definition}[Trajectory Comparison] \label{def:traj_cmp}
A trajectory $\mathcal{X}_1$ is said to be {\bf better} (less violating) than another trajectory $\mathcal{X}_2$
if the highest priority rule(s) violated by $\mathcal{X}_1$ has a lower priority than the highest priority rule(s) violated by $\mathcal{X}_2$. If both trajectories violate an equivalent highest priority rule(s), then the one with the smaller (maximum) total violation score is better. In this case, if the trajectories have equal violation scores, then they are equivalent.
\end{definition}
\vspace{-2pt}
It is easy to see that, by following Def. \ref{def:traj_cmp}, given two trajectories, one can be better than the other, or they can be equivalent (i.e., two trajectories cannot be incomparable). 

\begin{example}\label{ex:three-traj}
Consider the driving scenario from Fig. \ref{fig:autonomous1} and a priority structure $\langle R,\sim_p,\leq_p\rangle$ in Fig. \ref{fig:rulebook1}, where $R = \{r_1, r_2, r_3, r_4\}$,  and $r_1$: ``No collision'', $r_2$: ``Lane keeping'', $r_3$: ``Speed limit'' and $r_4$: ``Comfort''. There are 3 equivalence classes given by $\mathcal{O}_1=\{r_4\}$, $\mathcal{O}_2=\{r_2,r_3\}$ and $\mathcal{O}_3=\{r_1\}$. Rule $r_4$ has priority 1, $r_2$ and $r_3$ have priority 2, and $r_1$ has priority 3. Assume the instance (same as total, as there is only one instance for each rule) violation scores of rule $r=1,2,3,4$ by trajectories $a,b,c$ are given by $\rho_r=(\rho_r(a),\rho_r(b),\rho_r(c))$ as shown in Fig. \ref{fig:rulebook1}. 
Based on Def. \ref{def:traj_cmp}, 
trajectory $c$ is better (less violating) than trajectory $a$ since the highest priority rule violated by $c$ ($r_2$) has a lower priority than the highest priority rule violated by $a$ ($r_1$). The same argument holds for trajectories $a$ and $b$, i.e., $b$ is better than $a$. The highest priority rules violated by trajectories $b$ and $c$ have the same priorities. Since the maximum violation score of the highest priority rules violated by $b$ is smaller than that for $c$, i.e., $\max(\rho_2(b),\rho_3(b))=0.35$, $\max(\rho_2(c),\rho_3(c))=0.4$, trajectory $b$ is better than $c$. 
\end{example}
\begin{definition} (Priority structure satisfaction) \label{def:rb_satisfy}
A trajectory $\mathcal{X}$ of system (\ref{eqn:affine}) starting at $\bm x(0)$
satisfies a priority structure
$\langle R,\sim_p,\leq_p\rangle$ (i.e., $\mathcal{X}\models \langle R,\sim_p,\leq_p\rangle$), if there are no better trajectories of (\ref{eqn:affine}) starting at $\bm x(0)$.
\end{definition}

Def. \ref{def:rb_satisfy} is central to our solution to Problem \ref{prob:main} (see Sec. \ref{sec:optim-rule-approx}), which is based on an iterative relaxation of the rules according to their satisfaction of the priority structure. 


\begin{figure}[htb]
	\centering
	\subfigure[Possible trajectories]{
		\begin{minipage}[t]{0.45\linewidth}
			\centering
			\includegraphics[scale=0.28]{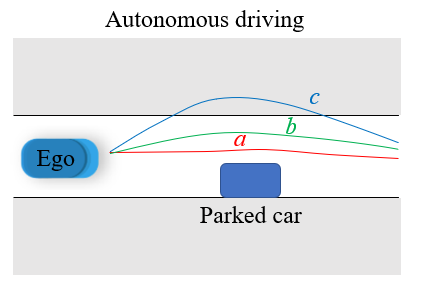} 
			\vspace{-1.8mm}
			\label{fig:autonomous1}%
		\end{minipage}%
	}	
	\subfigure[Priority structure with instance violation scores (the colors for the scores correspond to the colors of the trajectories. The rectangles show the equivalence classes.]{
		\begin{minipage}[t]{0.45\linewidth}
			\centering
			\includegraphics[scale=0.23]{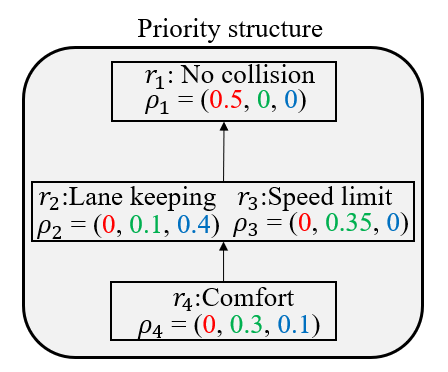} 
			\label{fig:rulebook1}%
		\end{minipage}%
	}	
	\centering
	\caption{An autonomous driving scenario with three possible trajectories, 4 rules, and 3 equivalence classes}
\end{figure} 

\section{RULE-BASED OPTIMAL CONTROL}
\label{sec:oc}
In this section, we present our approach to solve Problem \ref{prob:main}. 

\subsection{Trajectory Tracking}\label{sec:tracking}

As discussed in Sec. \ref{sec:vd}, Eqn. (\ref{eqn:affine}) can define ``traditional" vehicle dynamics with respect to an inertial reference frame
\cite{Ames2017}, or dynamics defined along a given reference trajectory \cite{Rucco2015} (see (\ref{eqn:vehicle})). The case study considered in this paper falls in the second category (the middle of ego's current lane is the default reference trajectory). We use the model from \cite{Rucco2015}, in which part of the state of (\ref{eqn:affine}) captures the tracking errors with respect to the reference trajectory. 
The tracking problem then becomes stabilizing the error states to 0. Suppose the error state vector is $\bm y\in{R}^{n_0}, n_0 \leq n$ (the components in $\bm y$ are part of the components in $\bm x$). We define a CLF $V(\bm x) = ||\bm y||^2$ ($c_3 = \epsilon > 0$ in Def. \ref{def:clf}). Any control $\bm u$ that satisfies the relaxed CLF constraint \cite{Ames2017} given by:
\begin{equation} \label{eqn:clf1}
     L_{f}V(\bm x)+L_{g}V(\bm x)
\bm u + \epsilon V(\bm x)\leq \delta_e,
\end{equation} 
exponentially stabilizes the error states to 0 if $\delta_e(t) = 0, \forall t\in[0,T]$, where $\delta_e>0$ is a relaxation variable that compromises between stabilization and feasibility. Note that the CLF constraint (\ref{eqn:clf1}) only works for $V(\bm x)$ with relative degree one. If the relative degree is larger than $1$, we can use input-to-state linearization and state feedback control \cite{Khalil2002} to reduce the relative degree to one \cite{Xiao2020}.

\subsection{Clearance and Optimal Disk Coverage} 
\label{sec:rule-approx}
Satisfaction of a priority structure can be enforced by formulating real-time constraints on ego state $\bm x(t)$ that appear in the violation metrics. Satisfaction of the non-clearance rules can be easily implemented using HOCBFs (See Sec. \ref{sec:optim-rule-approx}, Sec. \ref{sec:app-rule-def}). For clearance rules, we define a notion of clearance region around ego and around the traffic participants in $S_p$ that are involved in the rule (e.g., pedestrians and other vehicles). 
The clearance region for ego is defined as a rectangle with tunable speed-dependent lengths (i.e., we may choose to have a larger clearance from pedestrians when ego is driving with higher speeds
) and defined based on ego footprint and functions $h_f(\bm x), h_b(\bm x), h_l(\bm x), h_r(\bm x)$ that determine the front, back, left, and right clearances as illustrated in Fig. \ref{fig:approx}, where $h_f,h_b,h_l,h_r:\mathbb{R}^n\rightarrow \mathbb{R}_{\geq0}$. The clearance regions for participants (instances) are defined such that they comply with their geometry and cover their footprints, e.g., (fixed-length) rectangles for other vehicles and (fixed-radius) disks for pedestrians, as shown in Fig. \ref{fig:approx}.

To satisfy a clearance rule involving traffic participants, we need to avoid any overlaps between the clearance regions of ego and traffic participants. 
We define a function $d_{min}(\bm x, \bm x_i): \mathbb{R}^{n+n_i} \rightarrow \mathbb{R}$ to determine the signed distance between the clearance regions of ego and participant $i\in S_p$ ($\bm x_i\in\mathbb{R}^{n_i}$ denotes the state of participant $i$), which is negative if the clearance regions overlap. Therefore, satisfaction of a clearance rule can be imposed by having a constraint on $d_{min}(\bm x, \bm x_i)$ to be non-negative. For the clearance rules ``stay in lane" and ``stay in drivable area", we require that ego clearance region be within the lane and the drivable area, respectively.


However, finding $d_{min}(\bm x, \bm x_i)$ can be computationally expensive. For example, the distance between two rectangles could be from corner to corner, corner to edge, or edge to edge. Since each rectangle has $4$ corners and $4$ edges, there are 64 possible cases. More importantly, this computation leads to a non-smooth $d_{min}(\bm x, \bm x_i)$ function, which cannot be used to enforce clearance using a CBF approach. To address these issues, we propose an optimal coverage of the rectangles with disks, which allows to map the satisfaction of the clearance rules to a set of smooth HOCBF constraints (i.e., there will be one constraint for each pair of centers of disks pertaining to different traffic participants). 

We use $l > 0$ and $w > 0$ to denote the length and width of ego's footprint, respectively. Assume we use $z \in\mathbb{N}$ disks with centers located on the center line of the clearance region to cover it (see Fig. \ref{fig:proof}). Since all the disks have the same radius, the minimum radius to fully cover ego's clearance region, denoted by $ \mathfrak{r}>0$, is given by:
\begin{equation}\label{eqn:radius}
    \mathfrak{r} = \sqrt{\left(\frac{w + h_l(\bm x) + h_r(\bm x)}{2} \right)^2 + \left(\frac{l+h_f(\bm x) + h_b(\bm x)}{2z}\right)^2}.
\end{equation}
The minimum radius $\mathfrak{r}_i$ of the rectangular clearance region for a traffic participant $i \in S_p$ with disks number $z_i$ is defined in a similar way using the length and width of its footprint and setting $h_l,h_r,h_b,h_f=0$.
\begin{figure}[thpb]
	\centering
	\includegraphics[scale=0.30]{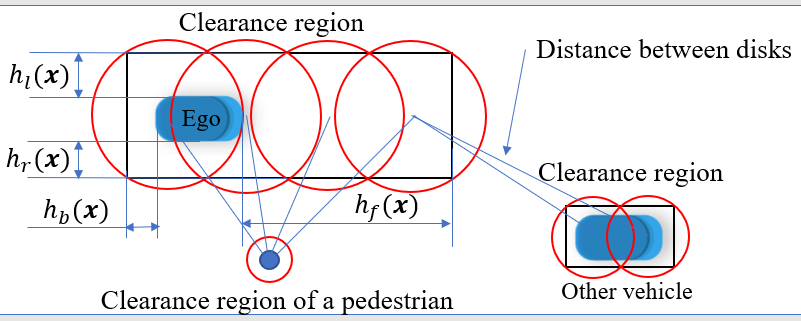}
	\caption{The clearance regions and their coverage with disks: the clearance region and the disks are speed dependent for ego and fixed for the other vehicle and the pedestrian. We consider the distances between all the possible pairs of disks from ego and other traffic participants (vehicles, pedestrians, etc.).}
	\label{fig:approx}%
\end{figure}
\begin{figure}[!bht]
	\centering
	\includegraphics[scale=0.30]{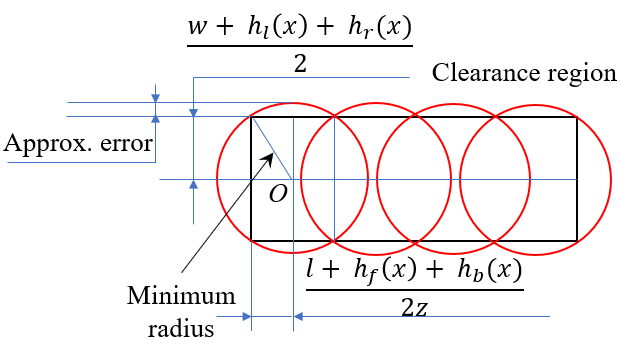}
	\vspace{-3pt}
	\caption{The optimal disk coverage of a clearance region.}
	\label{fig:proof}%
\end{figure}

Assume the center of the disk $j\in \{1,\dots,z\}$ for ego, and the center of the disk $k\in \{1,\dots,z_i\}$ for the instance $i \in S_p$ are given by $(x_{e,j}, y_{e,j}) \in \mathbb{R}^2$ and $(x_{i,k}, y_{i,k})\in \mathbb{R}^2$, respectively (See Appendix \ref{sec:app-coverage}). To avoid any overlap between the corresponding disks of ego and the instance $i\in S_p$, we impose the following constraints:
\begin{equation}\label{eqn:rule_cons}
\begin{aligned}
        \sqrt{(x_{e,j} - x_{i,k})^2 + (y_{e,j} - y_{i,k})^2} \geq   \mathfrak{r} + \mathfrak{r}_i ,\\ \forall j\in\{1,\dots,z\}, \forall k\in\{1,\dots,z_i\}. 
    \end{aligned}
\end{equation}
Since disks fully cover the clearance regions, enforcing \eqref{eqn:rule_cons} also guarantees that $d_{min}(\bm x, \bm x_i)\ge0$. For the clearance rules ``stay in lane" and ``stay in drivable area", we can get similar constraints as (\ref{eqn:rule_cons}) to make the disks that cover ego's clearance region stay within them (e.g., we can consider $h_l,h_r,h_b,h_f=0$ and formulate \eqref{eqn:rule_cons} such that the distance between ego disk centers and the line in the middle of ego's current lane be less than $\frac{w}{2} - \mathfrak{r}$). Thus, we can formulate satisfaction of all the clearance rules as continuously differentiable constraints (\ref{eqn:rule_cons}), and implement them using HOCBFs.

To efficiently formulate the proposed optimal disk coverage approach, we need to find the minimum number of the disks that fully cover the clearance regions as it determines the number of constraints in \eqref{eqn:rule_cons}. Moreover, we need to minimize the lateral approximation error since large errors imply overly conservative constraint (See Fig. \ref{fig:proof}). This can be formally defined as an optimization problem, and solved offline 
to determine the numbers and radii of the disks in \eqref{eqn:rule_cons} (the details are provided in Appendix \ref{sec:app-coverage}). 

\subsection{Optimal Control}
\label{sec:optim-rule-approx}

In this section, we present our complete framework to solve Problem \ref{prob:main}.  
We propose a recursive algorithm to iteratively relax the satisfaction of the rules in the priority structure $\langle R,\sim_p,\leq_p\rangle$ (if needed) based on the total order over the equivalence classes. 

Let $R_\mathcal{O}$ be the set of equivalence classes in $\langle R,\sim_p,\leq_p\rangle$, and $N_\mathcal{O}$ be the cardinality of $R_\mathcal{O}$. 
We construct the power set of equivalence classes denoted by $S = 2^{R_\mathcal{O}}$, and incrementally (from low to high priority) sort the sets in $S$ based on the highest priority of the equivalence classes in each set according to the total order and denote the sorted set by $S_{sorted} = \{S_1, S_2, \dots, S_{2^{N_\mathcal{O}}}\}$, where $S_1 =\{ \emptyset\}$. We use this sorted set in our optimal control formulation to obtain satisfaction of the higher priority classes, even at the cost of relaxing satisfaction of the lower priority classes. Therefore, from Def. \ref{def:rb_satisfy}, the solution of the optimal control will satisfy the priority structure. 
\begin{example}\label{exm:sorted}
Reconsider Exm. \ref{ex:three-traj}. 
We define $R_\mathcal{O} = \{ \mathcal{O}_1,\mathcal{O}_2,\mathcal{O}_3\}$. 
Based on the given total order $\mathcal{O}_1\leq_p \mathcal{O}_2 \leq_p \mathcal{O}_3$, we can write the sorted power set as $S_{sorted} = \left \{\right.\!\{\emptyset\}, \{\mathcal{O}_1\},\{\mathcal{O}_2\},\{\mathcal{O}_1,\mathcal{O}_2\},\{\mathcal{O}_3\},\\ \{\mathcal{O}_1,\mathcal{O}_3\},\{\mathcal{O}_2,\mathcal{O}_3\}, \{\mathcal{O}_1,\mathcal{O}_2,\mathcal{O}_3\} \}$. 
\end{example}

In order to find a trajectory that satisfies a given priority structure, we first assume that all the rules are satisfied. Starting from $S_1=\{\emptyset\}$ in the sorted set $S_{sorted}$, we solve Problem \ref{prob:main} given that no rules are relaxed, i.e., all the rules must be satisfied. If the problem is infeasible, we move to the next set $S_2 \in S_{sorted}$, and relax all the rules of all the equivalence classes in $S_2$ while enforcing satisfaction of all the other rules in the equivalence class set denoted by $R_\mathcal{O} \setminus S_2$. This procedure is done recursively until we find a feasible solution of Problem \ref{prob:main}. 
Formally, at $k = 1,2\dots, 2^{N_\mathcal{O}}$ for $S_k\in S_{sorted}$, we relax all the rules $i\in \mathcal{O}$ for all the equivalence classes $\mathcal{O} \in S_k$ and reformulate Problem \ref{prob:main} as the following optimal control problem:
\begin{equation}
\min_{\bm u,\delta_e, {\delta_i}_{,i\in S_k}} \int_{0}^{T}J(||\bm u||) +  p_e\delta_e^2 +\sum_{i\in S_k}p_i \delta_i^2dt \label{eqn:cost2}
\end{equation}
subject to:\\
\text{\qquad dynamics (\ref{eqn:affine}), control bounds (\ref{eqn:control}), CLF constraint (\ref{eqn:clf1}),}
\begin{align}
&\begin{aligned}L_{f}^{m_{j}}b_{j}(\bm x)+L_{g}L_{f}^{m_{j}-1}b_{j}(\bm x)\bm
u+S(b_{j}(\bm x))&\\+\alpha_{m_j}(\psi_{m_{j}-1}(\bm x))&\geq0, \forall j\in \mathcal{O}, \forall \mathcal{O}\in R_{\mathcal{O}}\setminus S_k,\end{aligned}\label{eqn:optim-relax-rules}
\\
&\begin{aligned}L_{f}^{m_{i}}b_{i}(\bm x)+L_{g}L_{f}^{m_{i}-1}b_{i}(\bm x)\bm
u+S(b_{i}(\bm x))&\\+\alpha_{m_i}(\psi_{m_{i}-1}(\bm x))&\geq \delta_i,\forall i\in \mathcal{O}, \forall \mathcal{O}\in S_k,\end{aligned}\label{eqn:optim-not-relax-rules}
\\
&\begin{aligned}
L_{f}^{m_{l}}b_{l}(\bm x)+L_{g}L_{f}^{m_{l}-1}b_{lim,l}(\bm x)\bm
u&+S(b_{lim,l}(\bm x))\\+\alpha_{m_l}(\psi_{m_{l}-1}(\bm x))&\geq0,\forall l\in \{1,\dots, 2n\},
\end{aligned} \label{eqn:optim-state-cons}
\end{align}
where $p_e > 0$ and $p_i>0, i\in S_k$ assign the trade-off between the the CLF relaxation $\delta_e$ (used for trajectory tracking) and the HOCBF relaxations $\delta_i$. 
$m_i,m_j,m_l$ denotes the relative degree of $b_i(\bm x),b_j(\bm x),b_{lim,l }(\bm x)$, respectively. The functions $b_i(\bm x)$ and $b_j(\bm x)$ are HOCBFs for the rules in $\langle R,\sim_p,\leq_p\rangle$, and are implemented directly from the rule statement for non-clearance rules or by using the optimal disk coverage framework for clearance rules. At relaxation step $k$, HOCBFs corresponding to the rules in $\mathcal{O}$, $\forall\mathcal{O}\in S_k$ are relaxed by adding $p_i>0, i\in S_k$ in \eqref{eqn:optim-relax-rules}, while for other rules in $R$ in \eqref{eqn:optim-not-relax-rules} and the state constraints \eqref{eqn:optim-state-cons}, regular HOCBFs are used. We assign $p_i, i\in S_k$ according to their relative priorities, i.e., we choose a larger $p_i$ for the rule $i$ that belongs to a higher priority class. 
The functions $b_{lim,l}(\bm x), l\in\{1,\dots,2n\}$ are HOCBFs for the state limitations (\ref{eqn:state}). The functions $\psi_{m_i}(\bm x), \psi_{m_j}(\bm x), \psi_{m_l}(\bm x)$ are defined as in (\ref{eqn:functions}). $\alpha_{m_i},\alpha_{m_j},\alpha_{m_l}$ can be penalized to improve the feasibility of the problem above \cite{Xiao2019,Xiao2020CDC}.

If the above optimization problem is feasible for all $t\in[0,T]$, we can specifically determine which rules (within an equivalence class) are relaxed based on the values of $\delta_i, i\in \mathcal{O}, \mathcal{O}\in S_k$ in the optimal solution (i.e., if $\delta_i(t) = 0, \forall t\in\{0,T\}$, then rule $i$ does not need to be relaxed). This procedure is summarized in Alg. \ref{alg:sort}. 
\begin{remark}[Complexity]
The optimization problem (\ref{eqn:cost2}) is solved using QPs introduced in Sec.~\ref{sec:pre}. The complexity of the QP is $O(y^3)$, where $y\in\mathbb{N}$ is the dimension of decision variables. It usually takes less than $0.01s$ to solve each QP in Matlab. The total time for each iteration $k\in\{1,\dots, 2^{N_{\mathcal{O}}}\}$ depends on the final time $T$ and the length of the reference trajectory $\mathcal{X}_r$. The computation time can be further improved by running the code in parallel over multiple processors.   
\end{remark}
\subsection{Pass/Fail Evaluation}\label{sec:p/f}
As an extension to Problem \ref{prob:main}, we formulate and solve a pass / fail (P/F) procedure, in which we are given a vehicle trajectory, and the goal is to accept (pass, P) or reject (fail, F) it based on the satisfaction of the rules. Specifically, given a candidate trajectory $\mathcal{X}_c$ of system (\ref{eqn:affine}), and given a priority structure $\langle R,\sim_p,\leq_p\rangle$, we pass (P) $\mathcal{X}_c$ if we cannot find a better trajectory according to Def. \ref{def:traj_cmp}. Otherwise, we fail (F) $\mathcal{X}_c$. 
We proceed as follows: We find the total violation scores of the rules in $\langle R,\sim_p,\leq_p\rangle$ for the candidate trajectory $\mathcal{X}_c$. If no rules in $R$ are violated, then we pass the candidate trajectory. Otherwise, 
we investigate the existence of a better (less violating) trajectory. We take the middle of ego's current lane as the reference trajectory $\mathcal{X}_r$
and re-formulate the optimal control problem in (\ref{eqn:cost2}) to recursively relax rules such that if the optimization is feasible, the generated trajectory is better than the candidate trajectory $\mathcal{X}_c$. Specifically, assume that 
the highest priority rule(s) that the candidate trajectory $\mathcal{X}_c$ 
violates belongs to $\mathcal{O}_H$, $H \in\mathbb{N}$. Let $R_H\subseteq R_{\mathcal{O}}$ denote the set of equivalence classes with priorities not larger than $H$, and $N_H \in\mathbb{N}$ denote the cardinality of $R_H$. We construct a power set $S_{H} = 2^{R_H}$, and then apply Alg. \ref{alg:sort}, in which we replace $R_{\mathcal{O}}$ by $R_H$. \vspace{-3pt}
\begin{remark}\label{remark:cond-pass}
The procedure described above would fail a candidate trajectory $\mathcal{X}_c$ even if only a slightly better alternate trajectory (i.e., violating rules of the same highest priority but with slightly smaller violation scores) can be found by solving the optimal control problem. In practice, this might lead to an undesirably high failure rate. One way to deal with this, which we will consider in future work (see Sec. \ref{sec:con}), is to allow for more classification categories, e.g., ``Provisional Pass" (PP), which can then trigger further investigation of $\mathcal{X}_c$.   
\end{remark}
\begin{example}
Reconsider Exm. \ref{ex:three-traj} and assume trajectory $b$ is a candidate trajectory which violates rules $r_2, r_4$, 
thus, the highest priority rule that is violated by trajectory $b$ belongs to $\mathcal{O}_2$. 
We construct $R_H = \{ \mathcal{O}_1,\mathcal{O}_2\}$. 
The power set $S_H=2^{R_H}$ is then defined as $S_H = \{ \{\emptyset\},\{\mathcal{O}_1\}, \{\mathcal{O}_2\},\{\mathcal{O}_1,\mathcal{O}_2\}\}$, and is sorted based on the total order as $S_{H_{sorted}} = \{\{\emptyset\}, \{\mathcal{O}_1\},\{\mathcal{O}_2\}, \{\mathcal{O}_1,\mathcal{O}_2\}\}$.
\end{example}
\begin{algorithm}
\caption{Recursive relaxation algorithm for finding optimal trajectory} \label{alg:sort}
	\KwIn{System (\ref{eqn:affine}) with $\bm x(0)$, cost function (\ref{eqn:gcost}), control bound (\ref{eqn:control}), state constraint (\ref{eqn:state}), priority structure $\langle R,\sim_p,\leq_p\rangle$, reference trajectory $\mathcal{X}_r$}
	\KwOut{Optimal ego trajectory and set of relaxed rules}
	1. Construct the power set of equivalence classes $S = 2^{R_\mathcal{O}}$\;
	2. Sort the sets in $S$ based on the highest priority of the equivalence classes in each set according to the total order and get $S_{sorted} = \{S_1, S_2, \dots, S_{2^{N_\mathcal{O}}}\}$\;
	3. $k = 0$\;
	\While{$k++\leq 2^{N_\mathcal{O}}$
	}{
    	Solve (\ref{eqn:cost2}) s.t. (\ref{eqn:affine}), (\ref{eqn:control}), (\ref{eqn:clf1}),  (\ref{eqn:optim-relax-rules}), (\ref{eqn:optim-not-relax-rules}) and (\ref{eqn:optim-state-cons})\; 
    	\If{the above problem is feasible for all $t\in[0,T]$}{
    	    Generate the optimal trajectory $\mathcal{X}^*$ from \eqref{eqn:affine}\;
        	Construct relaxed set $R_{relax} = \{i: i\in \mathcal{O}, \mathcal{O}\in S_{k}\}$\;
        	\If{$\delta_i(t) = 0, \forall t\in[0,T]$}{
        	    Remove $i$ from $R_{relax}$\;}
        	break\;
    	}
	}
	4. Return $\mathcal{X}^*$ and $R_{relax}$\;
\end{algorithm}

\section{Case Study}
\label{sec:case}



In this section, we apply the methodology developed in this paper to specific vehicle dynamics and various driving scenarios. Ego dynamics \eqref{eqn:affine} are defined with respect to a reference trajectory \cite{Rucco2015}, which measures the along-trajectory distance $s\in\mathbb{R}$ and the lateral distance $d\in\mathbb{R}$ of the vehicle Center of Gravity (CoG) with respect to the closest point on the reference trajectory as follows: 
\begin{equation} \label{eqn:vehicle}
   \underbrace{\left[
\begin{array}[c]{c}
    \dot s\\
    \dot d\\
    \dot \mu\\
    \dot v\\
    \dot a\\
    \dot \delta\\
    \dot \omega
\end{array}
\right]}_{\dot {\bm x}}
=
\underbrace{\left[
\begin{array}
[c]{c}%
    \frac{v\cos(\mu + \beta)}{1 - d\kappa}\\
    v\sin(\mu + \beta)\\
    \frac{v}{l_r}\sin\beta - \kappa\frac{v\cos(\mu + \beta)}{1 - d\kappa}\\
    a\\
    0\\
    \omega\\
    0
\end{array}
\right]}_{f(\bm x)}
+
\underbrace{\left[
\begin{array}[c]{cc}%
    0 & 0\\
    0 & 0\\
    0 & 0\\
    0 & 0\\
    1 & 0\\
    0 & 0\\
    0 & 1
\end{array}
\right]}_{g(\bm x)}
\underbrace{\left[
\begin{array}[c]{c}%
    u_{jerk}\\
    u_{steer}
\end{array}
\right]}_{\bm u},
\vspace{-2pt}
\end{equation}
where $\mu$ is the vehicle local heading error determined by the difference of the global vehicle heading $\theta\in\mathbb{R}$ in (\ref{eqn:center}) and the tangent angle $\phi\in\mathbb{R}$ of the closest point on the reference trajectory (i.e., $\theta = \phi + \mu$); $v$, $a$ denote the vehicle linear speed and acceleration; $\delta$, $\omega$ denote the steering angle and steering rate, respectively; $\kappa$ is the curvature of the reference trajectory at the closest point; $l_r$ is the length of the vehicle from the tail to the CoG; and $u_{jerk}$, $u_{steer}$ denote the two control inputs for jerk and steering acceleration as shown in Fig. \ref{fig:frame}. $\beta = \arctan\left(\frac{l_r}{l_r + l_f}\tan\delta\right)$ where $l_f$ is the length of the vehicle from the head to the CoG.\vspace{-2pt}
\begin{figure}[thpb]
	\centering
	\includegraphics[scale=0.3]{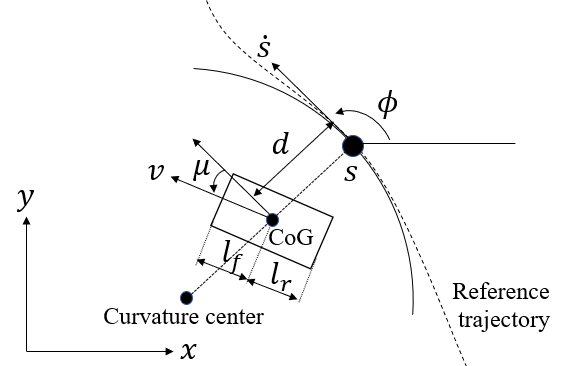}
	\caption{Coordinates of ego w.r.t a reference trajectory.}
	\label{fig:frame}
\end{figure}

We consider the cost function in \eqref{eqn:cost2} as:
\begin{equation}
     \min_{u_{jerk}(t), u_{steer}(t)}\int_{0}^{T}\left[u_{jerk}^2(t) + u_{steer}^2(t)\right]dt.
\end{equation}

The reference trajectory $\mathcal{X}_r$ is the middle of ego's current lane, and is assumed to be given as an ordered sequence of points $\bm p_1$, $\bm p_2$, $\dots$, $\bm p_{N_r}$, where $\bm p_i \in \mathbb{R}^2, i=1,\dots,N_r$ ($N_r$ denotes the number of points). We can find the reference point $p_{i(t)}$, $i:[0,T]\rightarrow \{1,\ldots,N_r\}$ at time $t$ as:
\vspace{-4pt}
\begin{equation}\label{eqn:tracking}
\begin{aligned}
    i(t)= \begin{cases} i(t) + 1 &  ||\bm p(t) - \bm p_{i(t)}||\leq \gamma,\\
     j & \exists j\in\{1,2,\dots, N_r\}:||\bm p(t) \!-\! \bm p_{i(t)}||\!\geq\! ||\bm p(t) \!-\! \bm p_{j}||,
\end{cases}
\end{aligned}
\end{equation}
where $\bm p(t)\in \mathbb{R}^2$ denotes ego's  location. $\gamma > 0$, and $i(0) = k$ for a $k\in\{1,2,\dots, N_r\}$ is chosen such that $||\bm p(0) - \bm p_{j}||\geq ||\bm p(0) - \bm p_{k}|, \forall j\in\{1,2,\, N_r\}$. Once we get $\bm p_{i(t)}$, we can update the progress $s$, the error states $d,\mu$ and the curvature $\kappa$ in (\ref{eqn:vehicle}).  The trajectory tracking in this case is to stabilize the error states $d, \mu$ ($\bm y = (d,\mu)$ in (\ref{eqn:clf1})) to 0, as introduced in Sec. \ref{sec:tracking}. We also wish ego to achieve a desired speed $v_d > 0$ (otherwise, ego may stop in curved lanes). We achieve this by re-defining the CLF $V(\bm x)$ in  (\ref{eqn:clf1}) as $V(\bm x) = ||\bm y||^2 + c_0(v-v_d)^2, c_0 > 0$.  As the relative degree of $V(\bm x)$ w.r.t. (\ref{eqn:vehicle}) is larger than 1, as mentioned in Sec. \ref{sec:tracking}, we use input-to-state linearization and state feedback control \cite{Khalil2002} to reduce the relative degree to one \cite{Xiao2020}. For example, for the desired speed part in the CLF $V(\bm x)$ ( (\ref{eqn:vehicle}) is in linear form from $v$ to $u_{jerk}$, so we don't need to do linearization), we can find a desired state feedback acceleration $\hat a = -k_1(v - v_d), k_1 > 0$. Then we can define a new CLF in the form $V(\bm x) = ||\bm y||^2 + c_0(a -\hat a)^2 = ||\bm y||^2 + c_0(a + k_1(v - v_d))^2$ whose relative degree is just one w.r.t. $u_{jerk}$ in (\ref{eqn:vehicle}). We proceed similarly for driving $d, \mu$ to 0 in the CLF $V(\bm x)$ as the relative degrees of $d, \mu$ are also larger than one.

The control bounds (\ref{eqn:control}) and state constraints (\ref{eqn:state}) are given by:
\begin{equation}\label{eqn:physical}
    \begin{aligned}
    &\text{speed constraint: } \qquad\quad\; v_{\min} \leq v(t)\leq v_{\max},\\
    &\text{acceleration constraint: } \;\; a_{\min}\leq a(t)\leq a_{\max},\\
    &\text{jerk control constraint: }\; u_{j,\min}\leq u_{jerk}(t)\leq u_{j,\max},\\
    &\text{steering angle constraint: } \delta_{\min}\leq \delta(t)\leq \delta_{\max},\\
    &\text{steering rate constraint: }\;\; \omega_{\min}\leq \omega(t)\leq \omega_{\max},\\
    &\text{steering control constraint: } u_{s,\min}\leq u_{steer}(t)\leq u_{s,\max},
    \end{aligned}
\end{equation}

We consider the priority structure $\langle R,\sim_p,\leq_p\rangle$ from Fig. \ref{fig:case_rb}, with rules $R = \{r_1, r_2, r_3, r_4, r_5, r_6, r_7, r_8\}$, where $r_1$ is a pedestrian clearance rule; $r_2$ and $r_3$ are clearance rules for staying in the drivable area and lane, respectively; $r_4$ and $r_5$ are non-clearance rules specifying maximum and minimum speed limits, respectively; $r_6$ is a comfort non-clearance rule; and $r_7$ and $r_8$ are clearance rules for parked and moving vehicles, respectively. The formal rule definitions (statements, violation metrics) are given in Appendix \ref{sec:app-rule-def}. These metrics are used to compute the scores for all the trajectories in the three scenarios below. 
The optimal disk coverage from Sec. \ref{sec:rule-approx} is used to compute the optimal controls for all the clearance rules, which are implemented using HOCBFs.
 
 \begin{figure}[thpb]
	\centering
	\includegraphics[scale=0.25]{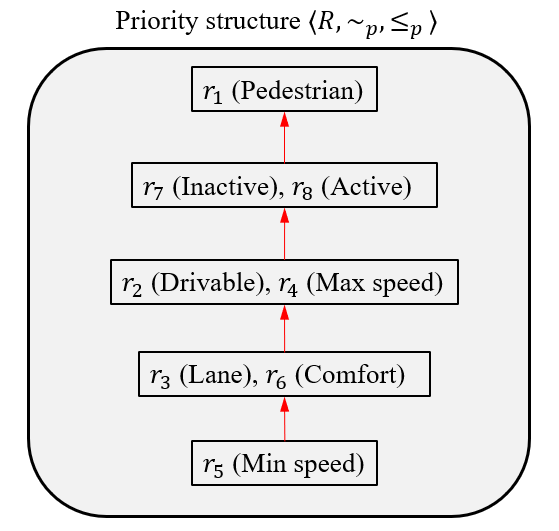}
	\vspace{-3mm}
	\caption{Priority structure for case study. 
	}
	\label{fig:case_rb}%
	\vspace{-3mm}
\end{figure}


In the following, we consider three common driving scenarios in our tool (See Appendix \ref{sec:tool}). For each of them, we solve the optimal control Problem \ref{prob:main} and perform pass/fail evaluation. In all three scenarios, in the pass/fail evaluation, an initial candidate trajectory is drawn ``by hand" using the tool described in the Appendix. We use CLFs to generate a feasible trajectory $\mathcal{X}_c$ which tracks the candidate trajectory subject to the vehicle dynamics (\ref{eqn:affine}), control bounds (\ref{eqn:control}) and state constraints (\ref{eqn:state}). 
\subsection{Scenario 1}
Assume there is an active vehicle, a parked (inactive) vehicle and a pedestrian, as shown in Fig. \ref{fig:case1}. 

\textbf{Optimal control:} 
We solve the optimal control problem (\ref{eqn:cost2}) by starting the rule relaxation from $S_1=\{\emptyset\}$ (i.e., without relaxing any rules). This problem is infeasible in the given scenario since ego cannot maintain the required distance between both the active and the parked vehicles as the clearance rules are speed-dependent. Therefore, we relaxed the next lowest priority equivalence class set in $S_{sorted}$, i.e., the minimum speed limit rule in $S_2=\{\{r_{5}\}\}$, for which we were able to find a feasible trajectory as illustrated in Fig. \ref{fig:case1}. By checking $\delta_i$ for $r_5$ from \eqref{eqn:cost2}, we found it is positive in some time intervals in $[0,T]$, and thus, $r_5$ is indeed relaxed. The total violation score for rule $r_{5}$ from \eqref{eqn:r_5} for the generated trajectory is 0.539, and all other rules in $R$ are satisfied. Thus, by Def. \ref{def:rb_satisfy}, the generated trajectory satisfies $\langle R,\sim_p,\leq_p\rangle$ in Fig. \ref{fig:case_rb}.

\begin{figure}[thpb]
	\centering
	\vspace{-3mm}
	\includegraphics[scale=0.5]{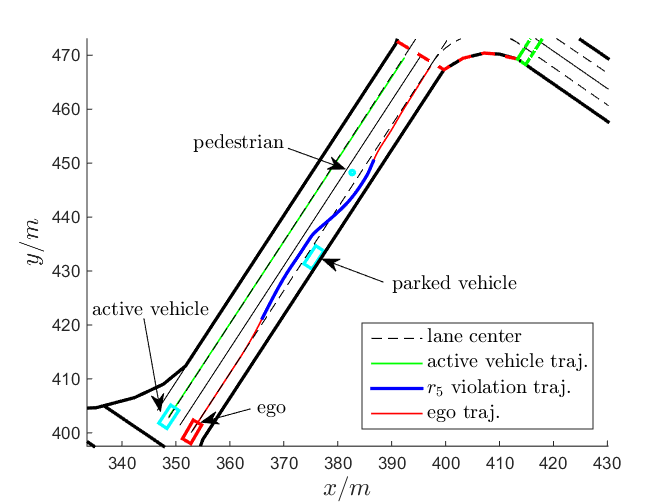}
	\vspace{-3mm}
	\caption{Optimal control for Scenario 1: the subset of ego trajectory violating $r_5$ is shown in blue.
	}
	\label{fig:case1}%
	\vspace{-3mm}
\end{figure}
\textbf{Pass/Fail:} The candidate trajectory $\mathcal{X}_c$ is shown in Fig. \ref{fig:case1_pf}. 
This candidate trajectory only violates rule $r_5$ with total violation score 0.682. Following Sec. \ref{sec:p/f}, we can either relax $r_5$ or do not relax any rules to find a possibly better trajectory. As shown in the above optimal control problem for this scenario, we cannot find a feasible solution if we do not relax rule $r_5$.
Since the violation score of the candidate trajectory is larger than the optimal one, we fail this candidate trajectory. 

\begin{figure}[thpb]
	\centering
	\includegraphics[scale=0.5]{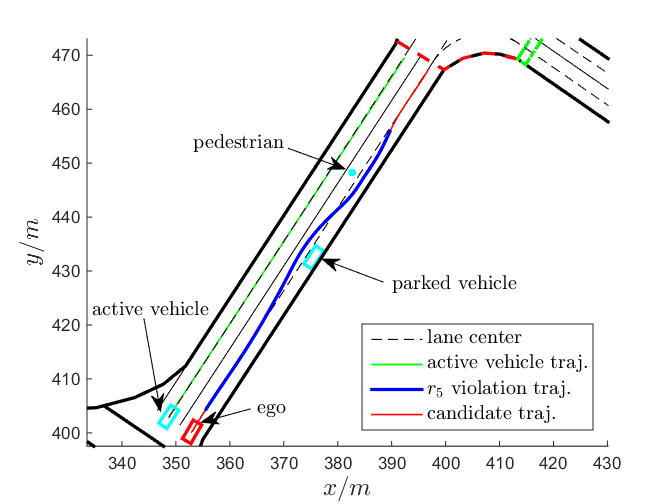}
	\vspace{-3mm}
	\caption{Pass/Fail for Scenario 1: the subset of the candidate trajectory violating $r_5$ is shown in blue.}
	\label{fig:case1_pf}
	\vspace{-3mm}
\end{figure}

\subsection{Scenario 2}
Assume there is an active vehicle, two parked  (inactive) vehicles and two pedestrians, as shown in Fig. \ref{fig:case2}. 

\textbf{Optimal control:} Similar to Scenario 1, the optimal control problem (\ref{eqn:cost2}) starting from $S_1=\{\emptyset\}$ (without relaxing any rules in $R$) is infeasible. We relax the next lowest priority rule set in $S_{sorted}$, i.e., the minimum speed rule in $S_2=\{\{r_{5}\}\}$, for which we are able to find a feasible trajectory as illustrated in Fig. \ref{fig:case2}. Again, the $\delta_i$ for $r_5$ is positive in some time intervals in $[0,T]$, and thus, $r_5$ is indeed relaxed. The total violation score of the rule $r_{5}$ for the generated trajectory is 0.646, and all the other rules in $R$ are satisfied. 

\begin{figure}[thpb]
	\centering
	\vspace{-1mm}
	\includegraphics[scale=0.5]{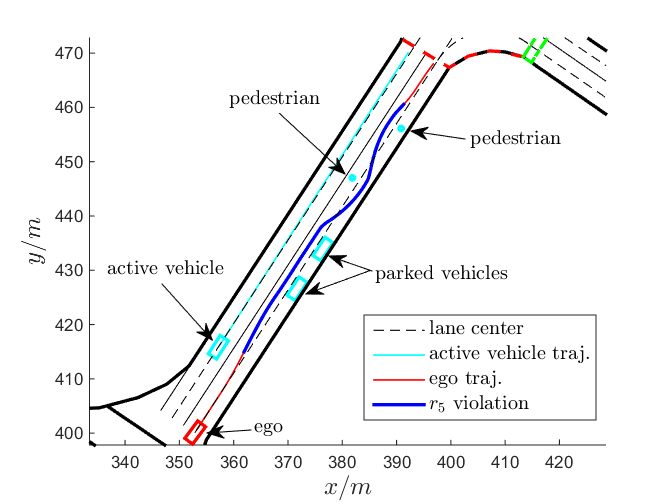}
	\vspace{-3mm}
	\caption{Optimal control for Scenario 2: the subset of ego trajectory violating $r_5$ is shown in blue.}
	\label{fig:case2}%
	\vspace{-3mm}
\end{figure}

\textbf{Pass/Fail:} The candidate trajectory $\mathcal{X}_c$ shown in red dashed line in Fig. \ref{fig:case2_pf} violates rules $r_1, r_{3}$ and $r_{8}$ with total violation scores 0.01, 0.23, 0.22 found from \eqref{eqn:r_1}, \eqref{eqn:r_3},\eqref{eqn:r_8}, 
respectively. 
In this scenario, we know that ego can change lane (where the lane keeping rule $r_3$ is in a lower priority equivalence class than $r_1$) to get reasonable trajectory. Thus, we show the case of relaxing the rules in the equivalence classes $\mathcal{O}_2=\{r_{3}, r_6\}$ and $\mathcal{O}_1=\{r_{5}\}$ to find a feasible trajectory that is better than the candidate one. The optimal control problem (\ref{eqn:cost2}) generates a trajectory as the red-solid curve shown in Fig. \ref{fig:case2_pf}, and only the $\delta_i$ for $r_6$ is 0 for all $[0,T]$. Thus, $r_6$ does not need to be relaxed. The generated trajectory violates rules $r_{3}$ and $r_{5}$ with total violation scores 0.124 and 0.111, respectively, but satisfies all the other rules including the highest priority rule $r_1$. According to Def. \ref{def:traj_cmp} for the given $\langle R,\sim_p,\leq_p\rangle$ in Fig. \ref{fig:case_rb}, the new generated trajectory is better than the candidate one, thus, we fail the candidate trajectory. Note that although this trajectory violates the lane keeping rule, it has a smaller violation score for $r_{5}$ compared to the trajectory obtained from the optimal control in Fig. \ref{fig:case2} (0.111 v.s. 0.646), i.e., the average speed of ego in the red-solid trajectory in Fig. \ref{fig:case2_pf} is larger.

\begin{figure}[thpb]
	\centering
	\vspace{-3mm}
	\includegraphics[scale=0.5]{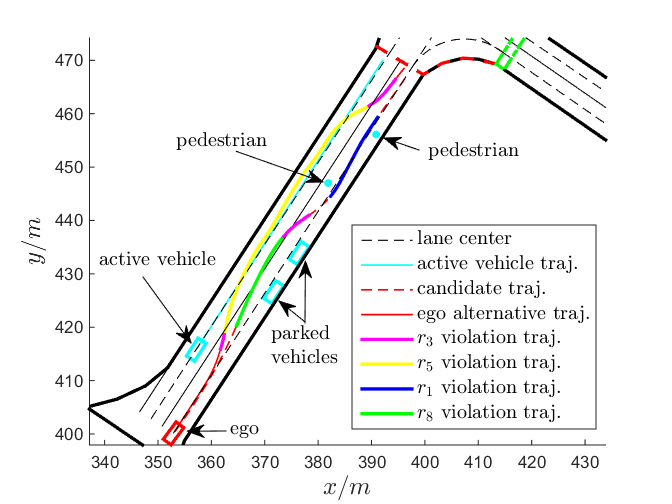}
	\vspace{-3mm}
	\caption{Pass/Fail for Scenario 2: the subsets of ego trajectory violating $r_5, r_3$ are shown in yellow and magenta, respectively; the subsets of the candidate trajectory violating $r_8, r_3, r_1$ are shown in green, magenta and blue, respectively.}
	\label{fig:case2_pf}
	\vspace{-3mm}
\end{figure}

\subsection{Scenario 3}
Assume there is an active vehicle, a parked vehicle and two pedestrians (one just gets out of the parked vehicle), as shown in Fig. \ref{fig:case3}. 

\textbf{Optimal control:} Similar to Scenario 1, the optimal control problem (\ref{eqn:cost2}) starting from $S_1=\{\emptyset\}$ (without relaxing any rules in $R$) is infeasible. We relax the lowest priority rule set in $S_{sorted}$, i.e., the minimum speed rule $S_2=\{\{r_{5}\}\}$, and solve the optimal control problem. In the (feasible) generated trajectory, ego stops before the parked vehicle, which satisfies all the rules in $R$ except $r_{5}$. Thus, by Def. \ref{def:rb_satisfy}, the generated trajectory satisfies the priority structure $\langle R,\sim_p,\leq_p\rangle$. However, this might not be a desirable behavior, thus, we further relax the lane keeping $r_{3}$ and comfort $r_6$ rules and find the feasible trajectory shown in Fig. \ref{fig:case3}. $\delta_i$ for $r_6$ is 0 for all $[0,T]$, and, therefore, $r_6$ does not need to be relaxed. The total violation scores for the rules $r_{3}$ and $r_{5}$ are 0.058 and 0.359, respectively, and all other rules in $R$ are satisfied.
\begin{figure}[thpb]
	\centering
	\includegraphics[scale=0.5]{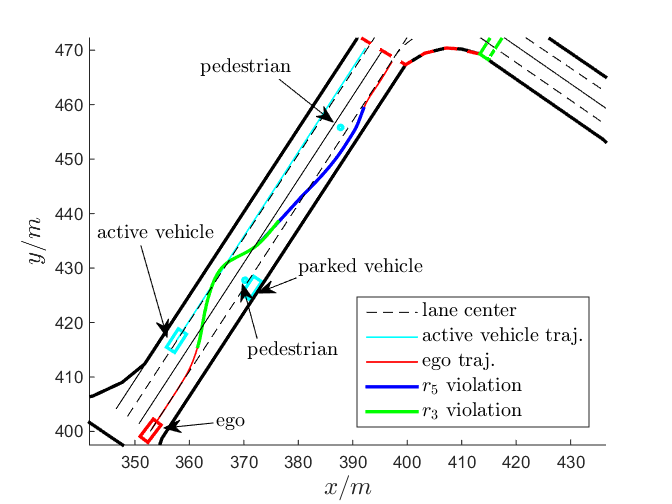}
	\vspace{-3mm}
	\caption{Optimal control for Scenario 3: the subsets of ego trajectory violating $r_5, r_3$ are shown in blue and green, respectively.}
	\label{fig:case3}%
	\vspace{-3mm}
\end{figure}

\textbf{Pass/Fail:} The candidate trajectory $\mathcal{X}_c$ shown as the red-dashed curve in Fig. \ref{fig:case3_pf} violates rules $r_{3}$ and $r_{8}$ with total violation scores 0.025 and 0.01, respectively. In this scenario, from the optimal control in Fig. \ref{fig:case3} we know that ego can change lane (where the lane keeping rule is in a lower priority equivalence class than $r_8$). We show the case of relaxing the rules in the equivalence classes $\mathcal{O}_2=\{r_{3}, r_6\}$ and $\mathcal{O}_1=\{r_{5}\}$ (all have lower priorities than $r_{8}$). The optimal control problem (\ref{eqn:cost2}) generates the red-solid curve shown in Fig. \ref{fig:case3_pf}. By checking $\delta_i$ for $r_6$, we found that $r_6$ is indeed not relaxed. The generated trajectory violates rules $r_{3}$ and $r_{5}$ with total violation scores 0.028 and 0.742, respectively, but satisfies all other rules including $r_{8}$. According to Def. \ref{def:traj_cmp} and Fig. \ref{fig:case_rb}, the new generated trajectory (although violates $r_{3}$ more than the candidate trajectory, it does not violate $r_{8}$ which has a higher priority) is better than the candidate one. Thus, we fail the candidate trajectory. 

\begin{figure}[thpb]
	\centering
	\includegraphics[scale=0.5]{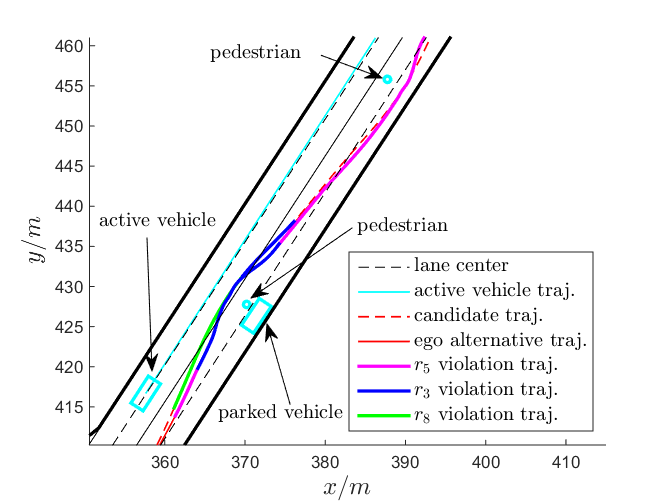}
	\vspace{-3mm}
	\caption{Pass/Fail for Scenario 3: the subsets of ego trajectory violating $r_8, r_5, r_3$ are shown in green, magenta and blue, respectively; the subsets of the candidate trajectory violating $r_5, r_3$ are shown in magenta and blue, respectively. 
	}
	\label{fig:case3_pf}
	\vspace{-3mm}
\end{figure}

\section{Conclusions and Future Work}
\label{sec:con}
We developed a framework to design optimal control strategies for autonomous vehicles that are required to satisfy a set of traffic rules with a given priority structure, while following a reference trajectory and satisfying control and state limitations. We showed that, for commonly used traffic rules, by using control barrier functions and control Lyapunov functions, the problem can be cast as an iteration of optimal control problems, where each iteration involves a sequence of quadratic programs. We also showed that the proposed algorithms can be used to pass / fail possible autonomous vehicle behaviors against prioritized traffic rules. We presented multiple case studies for an autonomous vehicle with realistic dynamics and conflicting rules. Future work will be focused on learning priority structures from data, improving the feasibility of the control problems, and refinement of the pass / fail procedure. 


\bibliographystyle{ACM-Reference-Format}
\bibliography{Biblio}


\newpage
\appendix
\section*{APPENDIX}
\section{Rule definitions}\label{sec:app-rule-def}
Here we give definitions for the rules used in Sec. \ref{sec:case}. According to Def. \ref{def:rule}, 
each rule statement should be satisfied for all times. 
\begin{equation}
    \begin{aligned} \label{eqn:r_1}
    r_1: &\text{ Maintain clearance with pedestrians}\\
    &\text{Statement: } d_{min,fp}(\bm x, \bm x_i)\geq d_{1} + v(t)\eta_1,\forall i\in S_{ped}\\
     &\varrho_{r,i}(\bm x(t)) = \max(0, \frac{d_{1} + v(t)\eta_1 - d_{min,fp}(\bm x, \bm x_i)}{d_{1} + v_{max}\eta_1})^2,\\
     &\rho_{r,i}(\mathcal{X}) = \max_{t\in[0,T]} \varrho_{r,i}(\bm x(t)), \quad
     P_r = \sqrt{\frac{1}{n_{ped}}\sum_{i\in S_{ped}}\rho_{r,i}}.
    \end{aligned}
\end{equation}
where $d_{min,fp}:\mathbb{R}^{n+n_i}\rightarrow \mathbb{R}$ 
denotes the distance between footprints of ego and the pedestrian $i$, and the clearance threshold is given based on a fixed distance $d_1\geq 0$ and increases linearly by $\eta_1 > 0$ based on ego speed $v(t) \geq 0$ ($d_1$ and $\eta_1$ are determined empirically), $S_{ped}$ denotes the index set of all pedestrians, and $\bm x_i\in \mathbb{R}^{n_i}$ denotes the state of the pedestrian $i$. $v_{max}$ is the maximum feasible speed of the vehicle and is used to define the normalization term in $\varrho_{r,i}$, which assigns a violation score (based on a L-$2$ norm) if formula is violated by $\bm x(t)$. $\rho_{r,i}$ defines the instance violation score as the most violating instant over $\mathcal{X}$. $P_r$ aggregates the instance violations over all units (pedestrians), where $n_{ped}\in\mathbb{N}$ denotes the number of pedestrians. 

\begin{equation}
    \begin{aligned} \label{eqn:r_2}
    r_{2}: &\text{ Stay in the drivable area}\\
    &\text{Statement: } d_{left}(\bm x(t)) + d_{right}(\bm x(t)) = 0\\
     &\varrho_r(\bm x(t)) = \left(\frac{d_{left}(\bm x(t)) + d_{right}(\bm x(t))}{2d_{max}}\right)^2,\\
     &\rho_r (\mathcal{X})= \sqrt{\frac{1}{T}\int_0^T \varrho_r(\bm x(t))dt}, \quad
     P_r = \rho_r.
    \end{aligned}
\end{equation}
where $d_{left}:\mathbb{R}^n\rightarrow\mathbb{R}, d_{right}:\mathbb{R}^n\rightarrow\mathbb{R}$ denote the left and right infringement distances of ego footprint into the non-drivable areas, respectively. $d_{\max} > 0$ denotes the maximum infringement distance and is used to normalize the instantaneous violation score defined based on a L-$2$ norm, and $\rho_r$ is the aggregation over trajectory duration $T$. 

\begin{equation}\label{eqn:r_3}
    \begin{aligned}
    r_{3}: &\text{ Stay in lane}\\
    &\text{Statement: } d_{left}(\bm x(t)) + d_{right}(\bm x(t)) = 0\\
     &\varrho_r(\bm x(t)) = \left(\frac{d_{left}(\bm x(t)) + d_{right}(\bm x(t))}{2d_{max}}\right)^2,\\
     &\rho_r (\mathcal{X})= \sqrt{\frac{1}{T}\int_0^T \varrho_r(\bm x(t))dt}, \quad
     P_r = \rho_r.
    \end{aligned}
\end{equation}
where $d_{left}:\mathbb{R}^n\rightarrow\mathbb{R}, d_{right}:\mathbb{R}^n\rightarrow\mathbb{R}$ denote the left and right infringement distances of ego footprint into the left and right lane boundaries, respectively, and violation scores are defined similar to the rule $r_{2}$.
\begin{equation}\label{eqn:r_4}
    \begin{aligned}
    r_{4}: &\text{ Satisfy the maximum speed limit}\\
    &\text{Statement: } v(t) \leq v_{max,s}\\
     &\varrho_r(\bm x(t)) = \max(0, \frac{v(t) - v_{max,s}}{v_{max}})^2,\\
     &\rho_r(\mathcal{X}) = \sqrt{\frac{1}{T}\int_0^T \varrho_r(\bm x(t))dt}, \quad
     P_r = \rho_r.
    \end{aligned}
\end{equation}
where $v_{max,s}>0$ denotes the maximum speed in a scenario $s$ and varies for different road types (e.g., highway, residential, etc.).
\begin{equation}\label{eqn:r_5}
    \begin{aligned}
    r_{5}: &\text{ Satisfy the minimum speed limit}\\
    &\text{Statement: } v(t) \geq v_{min,s}\\
     &\varrho_r(\bm x(t)) = \max(0, \frac{v_{min,s} - v(t)}{v_{min, s} - v_{min}})^2,\\
     &\rho_r(\mathcal{X}) = \sqrt{\frac{1}{T}\int_0^T \varrho_r(\bm x(t))dt}, \quad
     P_r = \rho_r.
    \end{aligned}
\end{equation}
where $v_{min,s}>0$ denotes the minimum speed in a scenario $s$ which varies for different road types and $v_{min}>0$ is the minimum feasible speed of the vehicle. 
\begin{equation}\label{eqn:r_6}
    \begin{aligned}
    r_{6}: &\text{ Drive smoothly}\\
    &\text{Statement: } |a(t)| \leq a_{max,s} \wedge |a_{lat}(t)| \leq a_{lat,s}\\
     &\varrho_r(\bm x(t)) = \left(\max(0, \frac{a_{max,s} - |a(t)|}{a_{max}}) + \max(0, \frac{a_{lat,s} - |a_{lat}(t)|}{a_{lat_{m}}})\right)^2,\\
     &\rho_r(\mathcal{X}) = \sqrt{\frac{1}{T}\int_0^T \varrho_r(\bm x(t))dt}, \quad
     P_r = \rho_r.
    \end{aligned}
\end{equation}
where $a_{lat}(t) = \kappa v^2(t)$ denotes the lateral acceleration at time instant $t$; $a_{max,s}>0,\; a_{lat,s}>0$ denote the maximum and the allowed lateral acceleration in a scenario $s$, respectively; and $a_{max}$ and $a_{lat_{m}}> 0$ denote the maximum feasible acceleration and maximum feasible lateral acceleration of the vehicle.
\begin{equation}\label{eqn:r_7}
    \begin{aligned}
    r_{7}: &\text{ Maintain clearance with parked vehicles}\\
    &\text{Statement: } d_{min,fp}(\bm x, \bm x_i)\geq d_{7} + v(t)\eta_7,\forall i\in S_{pveh}\\
     &\varrho_{r,i}(\bm x(t)) = \max(0, \frac{d_{7} + v(t)\eta_7 - d_{min,fp}(\bm x, \bm x_i)}{d_{7} + v_{max}\eta_7})^2,\\
     &\rho_{r,i}(\mathcal{X}) = \max_{t\in[0,T]} \varrho_i(\bm x(t)), \quad
     P_r = \sqrt{\frac{1}{n_{pveh}}\sum_{i\in S_{pveh}}\rho_{r,i}}
    \end{aligned}
\end{equation}
where $d_{min,fp}:\mathbb{R}^{n+n_i}\rightarrow \mathbb{R}$ denotes the distance between footprints of ego and the parked vehicle $i$, $d_7\geq 0,\; \eta_7 > 0$, and violation scores are defined similar to $r_{1}$, $S_{pveh}$ and $n_{pveh}\in\mathbb{N}$ denote the index set and number of parked vehicles, respectively, and $\bm x_i\in \mathbb{R}^{n_i}$ denotes the state of the parked vehicle $i$. 

\begin{equation}\label{eqn:r_8}
    \begin{aligned}
    r_{8}: &\text{ Maintain clearance with active vehicles}\\
    &\text{Statement: } \quad d_{min,l}(\bm x, \bm x_i)\geq d_{8,l} + v(t)\eta_{8,l} \\
    &\qquad\qquad \wedge  \; d_{min,r}(\bm x, \bm x_i)\geq d_{8,r} + v(t)\eta_{8,r}  \\
    &\qquad\qquad \wedge  \; d_{min,f}(\bm x, \bm x_i)\geq d_{8,f} + v(t)\eta_{8,f},\forall i\in S_{aveh}\\
     &\varrho_{r,i}(\bm x(t)) = \frac{1}{3}(\max(0, \frac{d_{8,l} + v(t)\eta_{8,l} - d_{min,l}(\bm x, \bm x_i)}{d_{8,l} + v_{max}\eta_{8,l}})^2 \\&\qquad\qquad+ \max(0, \frac{d_{8,r} + v(t)\eta_{8,r} - d_{min,r}(\bm x, \bm x_i)}{d_{8,r} + v_{max}\eta_{8,r}})^2 \\&\qquad\qquad+ \max(0, \frac{d_{8,f} + v(t)\eta_{8,f} - d_{min,f}(\bm x, \bm x_i)}{d_{8,f} + v_{max}\eta_{8,f}})^2),\\
     &\rho_{r,i}(\mathcal{X}) = \frac{1}{T}\int_0^T  \varrho_{r,i}(\bm x(t))dt, \quad
     P_r = \sqrt{\frac{1}{n_{aveh}-1}\sum_{i\in S_{aveh}\setminus {ego}}\rho_{r,i}}
    \end{aligned}
\end{equation}
where $d_{min, l}:\mathbb{R}^{n+n_i}\rightarrow \mathbb{R},d_{min, r}:\mathbb{R}^{n+n_i}\rightarrow \mathbb{R},d_{min, f}:\mathbb{R}^{n+n_i}\rightarrow \mathbb{R}$ denote the distance between footprints of ego and the active vehicle $i$ on the left, right and front, respectively; $d_{8,l}\geq 0,d_{8,r}\geq 0,d_{8,f}\geq 0, \eta_{8,l} > 0,\eta_{8,r} > 0,\eta_{8,f} > 0$ are defined similarly as in  $r_{1}$, and $S_{aveh}$ and $n_{aveh}\in\mathbb{N}$ denote the index set and number of active vehicles, and $\bm x_i\in \mathbb{R}^{n_i}$ denotes the state of the active vehicle $i$. Similar to Fig. \ref{fig:approx}, we show in Fig. \ref{fig:rule8} how $r_8$ is defined based on the clearance region and optimal disk coverage proposed in Sec. \ref{sec:app-coverage}.
\begin{figure}[thpb]
	\centering
	\vspace{-2mm}
	\includegraphics[scale=0.25]{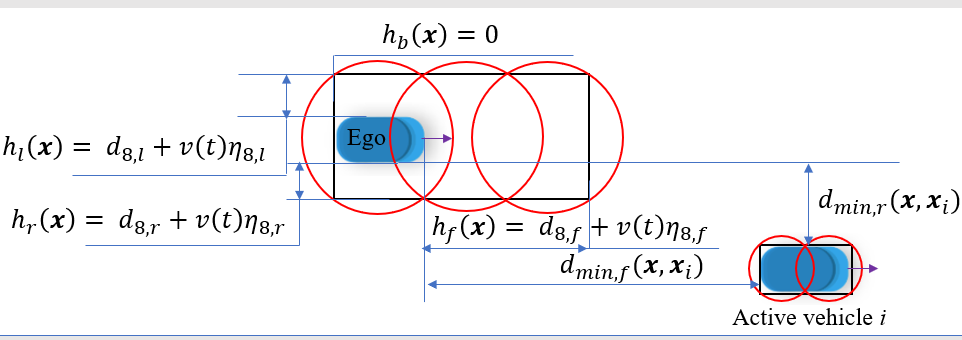}
	\vspace{-4mm}
	\caption{Formulation of $r_8$ with the optimal disk coverage approach: $r_8$ is satisfied since clearance regions of ego and the active vehicle $i\in S_p$ do not overlap.}
 	\label{fig:rule8}%
 	\vspace{-3mm}
\end{figure}
\section{Optimal Disk Coverage}\label{sec:app-coverage}


To construct disks to fully cover the clearance regions, we need to find their number and radius. From Fig. \ref{fig:proof}, the lateral approximation error $\sigma >0$ is given by:
\begin{equation}\label{eq:sigma}
    \sigma = \mathfrak{r} - \frac{w + h_l(\bm x) + h_r(\bm x)}{2}.
\end{equation}
Since $\sigma$ for ego depends on its state $\bm x$ (speed-dependent), we consider the accumulated lateral approximation error for all possible $\bm x\in X$. This allows us to determine $z$ and $\mathfrak{r}$ such that the disks fully cover ego clearance region for all possible speeds in $\bm x$. Let $\bar h_i = \sup_{\bm x\in X}h_i(\bm x), \underline{h}_i = \inf_{\bm x\in X}h_i(\bm x), i\in\{f,b,l,r\}$. We can formally formulate the construction of the approximation disks as an optimization problem:

\begin{equation}\label{eqn:opcircle}
    \min_{z} z + \beta \int_{\underline{h}_f}^{\bar h_f}\int_{\underline{h}_b}^{\bar h_b}\int_{\underline{h}_l}^{\bar h_l}\int_{\underline{h}_r}^{\bar h_r}\sigma dh_f(\bm x)dh_b(\bm x)dh_l(\bm x)dh_r(\bm x)
\end{equation}
subject to
\begin{equation}
    \begin{aligned}
    z \in \mathbb{N},
    \end{aligned}
\end{equation}
where $\beta\geq 0$ is a trade-off between minimizing the number of the disks (so as to minimize the number of constraints considered with CBFs) and the coverage approximation error. The above optimization problem is solved offline. 
A similar optimization is formulated for construction of disks for instances in $S_p$ (we remove the integrals due to speed-independence). Note that for the driving scenarios studied in this paper, we omit the longitudinal approximation errors in the front and back. The lateral approximation errors are considered in the disk formulation since they induce conservativeness in the lateral maneuvers of ego required for surpassing other instances (such as parked car, pedestrians, etc.), see Sec. \ref{sec:case}.

Let $(x_e,y_e) \in\mathbb{R}^2$ be the center of ego and $(x_i,y_i) \in\mathbb{R}^2$ be the center of the instance $i\in S_p$. The center of the disk $j$ for ego $(x_{e,j}, y_{e,j}), j\in \{1,\dots,z\}$ is determined by:
\begin{equation} \label{eqn:center}
\begin{aligned}
    x_{e,j} = x_e + \cos\theta_e(-\frac{l}{2} - h_b(\bm x) + \frac{l+h_f(\bm x) + h_b(\bm x)}{2z}(2j - 1))\\
    y_{e,j} = y_e + \sin\theta_e(-\frac{l}{2} - h_b(\bm x) + \frac{l+h_f(\bm x) + h_b(\bm x)}{2z}(2j - 1))
\end{aligned}
\end{equation}
where $j\in\{1,\dots, z\}$ and $\theta_e\in\mathbb{R}$ denotes the heading angle of ego. The center of the disk $k$ for the instance $i \in S_p$ denoted by $(x_{i,k}, y_{i,k}), k\in \{1,\dots,z_i\}$ can be defined similarly. 

\begin{theorem} \label{thm:cover}
If the clearance regions of ego and the instance $i\in S_p$ are covered by the disks constructed by solving (\ref{eqn:opcircle}), then satisfaction of (\ref{eqn:rule_cons}) implies non-overlapping of the clearance regions between ego and the instance $i$.
\end{theorem}
\begin{proof}
Consider $z$ and $z_i$ disks with minimum radius $\mathfrak{r}$ and $\mathfrak{r}_i$ from (\ref{eqn:radius}) associated with clearance region of ego and instance $i\in S_p$, respectively. The constraints in (\ref{eqn:rule_cons}) guarantee that there is no overlap of the disks between the vehicle $i\in S_p$ and instance $j\in S_p$. Since the clearance regions are fully covered by these disks, we conclude that the clearance regions do not overlap.
\end{proof}


\section{Software tool and simulation parameters}
\label{sec:tool}

We implemented the computational procedure described in this paper as a user-friendly 
software tool in Matlab. The tool allows to load a map represented by a .json file and place vehicles and pedestrians on it. It provides an interface to generate smooth reference / candidate trajectories and it implements our proposed optimal control and P/F frameworks; the $quadprog$ optimizer was used to solve the QPs (solve time $<0.01s$ for each QP) and $ode45$ to integrate the vehicle dynamics (\ref{eqn:vehicle}). All the computation in this paper was performed on a Intel(R) Core(TM) i7-8700 CPU @
3.2GHz$\times 2$.

The simulation parameters are considered as follows: $v_{max} = 10m/s, v_{min} = 0m/s, a_{max} = -a_{min} = 3.5m/s^2, u_{j,max} = -u_{j,min} = 4m/s^3, \delta_{max} = -\delta_{min} = 1rad, \omega_{max} = -\omega_{min} = 0.5rad/s, u_{s,max} = -u_{s,min} = 2rad/s^2, w = 1.8m, l = 4m, l_f = l_r = 2m, d_1 = 1m, \eta_1 = 0.067s, v_{max,s} = 7m/s, v_{min,s} = 3m/s, a_{max,s} = 2.5m/s^2, a_{lat_{m}} = 3.5m/s^2, a_{lat,s} = 1.75m/s^2, d_7 = 0.3m, \eta_7 = 0.13s, d_{8,l} = d_{8,r} = 0.5m, d_{8,f} = 1m, \eta_{8,r} = \eta_{8,l} = 0.036s, \eta_{8,f} = 2s, v_d = 4m/s, \beta = 2$ in (\ref{eqn:opcircle}). 
\end{document}